\documentclass{article}

\usepackage{microtype}
\usepackage{graphicx}
\usepackage{subfigure}
\usepackage{booktabs} 

\usepackage{hyperref}

\usepackage{microtype}
\usepackage{graphicx}
\usepackage{subfigure}
\usepackage{booktabs} 
\usepackage{dsfont}
\usepackage{amsthm}
\usepackage{thmtools,thm-restate}
\usepackage[toc,page]{appendix}
\usepackage[none]{hyphenat}
\usepackage{amsfonts}
\usepackage{amsmath}

\usepackage{bbold}
\usepackage{algorithm}
\usepackage{algorithmic}
\usepackage{graphicx}
\usepackage{subfigure} 
\newcommand{\norm}[1]{\left\lVert#1\right\rVert}
\newcommand{\Thatn}[0]{\widehat{T}}

\newcommand{\That}[2]{\widehat{T}^{#1}(#2)}

\newcommand{\fancyS}[0]{\mathcal{S}}
\newcommand{\fancyA}[0]{\mathcal{A}}

\newcommand{\real}[0]{\mathbb R}

\newtheorem*{theorem*}{Theorem}
\newtheorem{defn}{Definition}

\usepackage[accepted,nohyperref]{icml2018}

\icmltitlerunning{Equivalence Between Wasserstein and Value-Aware Model-based Reinforcement Learning}

\begin{document}

\twocolumn[
\icmltitle{Equivalence Between Wasserstein and Value-Aware Loss for\\ Model-based Reinforcement Learning}



\icmlsetsymbol{equal}{*}

\begin{icmlauthorlist}
\icmlauthor{Kavosh Asadi}{br}
\icmlauthor{Evan Cater}{br}
\icmlauthor{Dipendra Misra}{co}
\icmlauthor{Michael L. Littman}{br}
\end{icmlauthorlist}

\icmlaffiliation{br}{Department of Computer Science, Brown University, Providence, USA}
\icmlaffiliation{co}{Department of Computer Science, Cornell Tech, New York, USA}
\icmlcorrespondingauthor{Kavosh asadi}{kavosh@brown.edu}

\icmlkeywords{Machine Learning, ICML}

\vskip 0.3in
]



\printAffiliationsAndNotice{}  

\begin{abstract}
Learning a generative model is a key component of model-based reinforcement learning. Though learning a good model in the tabular setting is a simple task, learning a useful model in the approximate setting is challenging. In this context, an important question is the loss function used for model learning as varying the loss function can have a remarkable impact on  effectiveness of planning. Recently \citet{farahmand17} proposed a value-aware model learning (VAML) objective that captures the structure of value function during model learning.  Using tools from \citet{asadi2018lipschitz}, we show that minimizing the VAML objective is in fact equivalent to minimizing the Wasserstein metric. This equivalence improves our understanding of value-aware models, and also creates a theoretical foundation for applications of Wasserstein in model-based reinforcement~learning.
\end{abstract}

\section{Introduction}
The model-based approach to reinforcement learning consists of learning an internal model of the environment and planning with the learned model \citep{sutton98}. The main promise of the model-based approach is data-efficiency: the ability to perform policy improvements with a relatively small number of environmental interactions.

Although the model-based approach is well-understood in the tabular case \citep{KaelblingLM96,sutton98}, the extension to approximate setting is difficult. Models usually have non-zero generalization error due to limited training samples. Moreover, the model learning problem can be unrelizable, leading to an imperfect model with irreducible error \citep{RossB12,Talvitie14}. Sometimes referred to as the compounding error phenomenon, it has been shown that such small modeling errors can also compound after multiple steps and degrade the policy learned using the model \citep{Talvitie14,venkatraman2015improving, asadi2018lipschitz}.

On way of addressing this problem is by learning a model that is tailored to the specific planning algorithm we intend to use. That is, even though the model is imperfect, it is useful for the planning algorithm that is going to leverage it. To this end, \citet{farahmand17} proposed an objective function for model-based RL that captures the structure of value function during model learning to ensure that the model is useful for Value Iteration. Learning a model using this loss, known as value-aware model learning (VAML) loss, empirically improved upon a model learned using maximum-likelihood objective, thus providing a promising direction for learning useful models in the approximate setting.

More specifically, VAML minimizes the maximum Bellman error given the learned model, MDP dynamics, and an arbitrary space of value functions. As we will show, computing the Wasserstein metric involves a similar maximization problem, but over a space of Lipschitz functions. Under certain assumptions, we prove that the value function of an MDP is Lipschitz. Therefore, minimizing the VAML objective is in fact equivalent to minimizing Wasserstein.
\label{submission}
\section{Background}
\subsection{MDPs}
We consider the Markov decision process (MDP) setting in which the RL problem is formulated by the tuple $\langle \mathcal {S}, \mathcal{A},R,T,\gamma\rangle
$. Here, $\mathcal{S}$ denotes a state space and $\mathcal{A}$ denotes an action set. The functions $R:S\times A\rightarrow\mathbb R$ and
$T\mathcal{:S\times\ A \rightarrow}\  \textrm{Pr}(\fancyS)$ denote the reward and transition dynamics. Finally $\gamma \in [0,1)$ is the
discount rate.
\subsection{Lipschitz Continuity}
We make use of the notion of ``smoothness'' of a function as quantified below.
\begin{defn}
Given two metric spaces $(M_1,d_1)$ and $(M_2,d_2)$ consisting of a space and a distance metric, a function $f:M_1\mapsto M_2$ is \emph{Lipschitz continuous} (sometimes simply Lipschitz) if the Lipschitz constant, defined as
\begin{equation}
	K_{d_1,d_2}(f):=\sup_{s_1\in \fancyS,s_2\in \fancyS}\frac{d_2\big(f(s_1),f(s_2)\big)}{d_1(s_1,s_2)}\ ,
\end{equation} 
is finite.
\end{defn}

Equivalently, for a Lipschitz $f$, $$\forall s_1,\forall s_2 \quad d_2\big(f(s_1),f(s_2)\big)\leq K_{d_1,d_2}(f)\ d_1(s_1,s_2) \ \ . $$

Note that the input and output of $f$ can generally be scalars, vectors, or probability distributions. A Lipschitz function $f$ is called a \emph{non-expansion} when $K_{d_1,d_2}(f)=1$ and a \emph{contraction} when $K_{d_1,d_2}(f)< 1$. We also define Lipschitz continuity over a subset of inputs:
\begin{defn}
A function $f:M_1 \times \fancyA \mapsto M_2$ is \emph{uniformly Lipschitz continuous} in $\fancyA$ if  \begin{equation}
	K_{d_1,d_2}^\fancyA(f):=\sup_{a\in \fancyA}\sup_{s_1,s_2}\frac{d_2\big(f(s_1,a),f(s_2,a)\big)}{d_1(s_1,s_2)}\ ,
\end{equation}
is finite.
\end{defn}	
Note that the metric $d_1$ is still defined only on $M_1$.
Below we also present two useful lemmas.
\begin{restatable}{lemma}{compositionlemma}
(Composition Lemma) Define three metric spaces $(M_1,d_1)$, $(M_2,d_2)$, and $(M_3,d_3)$. Define Lipschitz functions $f:M_2\mapsto M_3$ and $g:M_1\mapsto M_2$ with constants $K_{d_2,d_3}(f)$ and $K_{d_1,d_2}(g)$. Then, $h:f \circ g:M_1 \mapsto M_3$ is Lipschitz with constant $K_{d_1,d_3}(h)\leq K_{d_2,d_3}(f) K_{d_1,d_2}(g)$.
\label{lemma_Lipschitz_Composition}
\end{restatable}
\begin{proof}
	\begin{eqnarray*}
			&&K_{d_1,d_3}(h)\\
			&=&\sup_{s1,s_2}\frac{d_{3}\Big(f\big(g(s_1)\big),f\big(g(s_2)\big)\Big)}{d_{1}(s_1,s_2)}\\
			&=&\sup_{s_1,s_2}\frac{d_{2}\big(g(s_1),g(s_2)\big)}{d_{1}(s_1,s_2)}\frac{d_{3}\Big(f\big(g(s_1)\big),f\big(g(s_2)\big)\Big)}{d_{2}\big(g(s_1),g(s_2)\big)}\\
			&\leq&\sup_{s_1,s_2}\frac{d_{2}\big(g(s_1),g(s_2)\big)}{d_{1}(s_1,s_2)}
			\sup_{s_1,s_2}\frac{d_{3}\big(f(s_1),f(s_2)\big)}{d_{2}(s_1,s_2)}\\
			&=&K_{d_1,d_2}(g)K_{d_2,d_3}(f).
	\end{eqnarray*}
\end{proof}
\begin{restatable}{lemma}{summationlemma}
(Summation Lemma) Define two vector spaces $(M_1,\norm{})$ and $(M_2,\norm{})$. Define Lipschitz functions $f:M_1\mapsto M_2$ and $g:M_1\mapsto M_2$ with constants $K_{\norm{},\norm{}}(f)$ and $K_{\norm{},\norm{}}(g)$. Then, $h:f + g:M_1 \mapsto M_2$ is Lipschitz with constant $K_{\norm{},\norm{}}(h)\leq K_{\norm{},\norm{}}(f)+K_{\norm{},\norm{}}(g)$.
\label{lemma_Lipschitz_summation}
\end{restatable}
\begin{proof}
	\begin{eqnarray*}
			K_{d_1,d_2}(h)&:=&\sup_{s_1,s_2}\frac{\norm{f(s_2)+g(s_2)-f(s_1)-g(s_1)}}{\norm{s_2-s_1}}\\
			&\leq&\sup_{s_1,s_2}\frac{\norm{f(s_2)-f(s_1)}}{\norm{s_2-s_1}}+\frac{\norm{g(s_2)-g(s_1)}}{\norm{s_2-s_1}}\\
			&\leq&\sup_{s_1,s_2}\frac{\norm{f(s_2)-f(s_1)}}{\norm{s_2-s_1}}\\
			&+&\sup_{s_1,s_2}\frac{\norm{g(s_2)-g(s_1)}}{\norm{s_2-s_1}}\\
			&=& K_{\norm{}{},\norm{}}(f)+K_{\norm{},\norm{}}(g)\\
	\end{eqnarray*}
\end{proof}
\subsection{Distance Between Distributions}
We require a notion of difference between two distributions quantified below.
\begin{defn}
Given a metric space $(M,d)$ and the set $\mathbb P(M)$ of all probability measures on $M$, the \emph{Wasserstein metric} (or the 1st Kantorovic metric) between two probability distributions $\mu_1$ and $\mu_2$ in $\mathbb P(M)$ is defined as
\begin{equation}
	W(\mu_1,\mu_2):=\inf_{j \in \Lambda}\int\!\int j(s_1,s_2) d(s_1,s_2)ds_2\ ds_1 \ ,
\end{equation}
where $\Lambda$ denotes the collection of all joint distributions $j$ on $M \times M$ with marginals $\mu_1$ and $\mu_2$~\citep{vaserstein1969markov}.
\end{defn}
Wasserstein is linked to Lipschitz continuity using duality:
\begin{equation}
	W(\mu_1,\mu_2)=\sup_{f:K_{d,d_{\real}}(f)\leq 1}\int\! \big(\mu_{1}(s)-\mu_{2}(s)\big)f(s)ds\ .
\label{KR_duality}
\end{equation}
This equivalence is known as Kantorovich-Rubinstein duality \citep{kantorovich1958space,villani2008optimal}. Sometimes referred to as ``Earth Mover's distance'', Wasserstein has recently become popular in machine learning, namely in the context of generative adversarial networks~\citep{arjovsky2017wasserstein} and value distributions in reinforcement learning~\citep{bellemare2017distributional}. We also define Kullback Leibler divergence (simply KL) as an alternative measure of difference between two distributions:
$$KL(\mu_1 \mid\mid \mu_2) := \int \mu_1(s) \log \frac{\mu_{1}(s)}{\mu_2(s)}ds\ .$$ 

\section{Value-Aware Model Learning (VAML) Loss}
\label{VAML_section}
The basic idea behind VAML \citep{farahmand17} is to learn a model tailored to the planning algorithm that intends to use it. Since Bellman equations \citep{bellman1957markovian} are in the core of many RL algorithms \citep{sutton98}, we assume that the planner uses the following Bellman equation:
$$Q(s,a)= R(s,a)+\gamma\int T(s'|s,a) f\big(Q(s',.)\big)ds' \ ,$$
where $f$ can generally be any arbitrary operator \citep{littman1996generalized} such as max. We also define:
$$v(s'):=f\big(Q(s',.)\big)\ .$$ 
A good model $\widehat{T}$ could then be thought of as the one that minimizes the error:
\begin{eqnarray*}
    l(T, \widehat T)(s, a)&=&R(s,a)+\gamma\int T(s'|s,a) v(s')ds' \nonumber\\
    &-& R(s,a)-\gamma\int \widehat T(s'|s,a) v(s')ds'\nonumber \\
    &=&\!\gamma\!\int\!\big( T(s'|s,a)\!-\!\That{}{s'|s,a}\big) v(s')ds'
    \label{original_obj}
\end{eqnarray*}

Note that minimizing this objective requires access to the value function in the first place, but we can obviate this need by leveraging Holder's inequality:
\begin{eqnarray*} 
l(\widehat T, T)(s, a) &=&\gamma \int  \big(T(s'|s,a)-\That{}{s'|s,a} \big) v(s')ds' \\ 
&\leq& \gamma \norm{T(s'|s,a)-\That{}{s'|s,a}}_{1}\norm{v}_{\infty}
\end{eqnarray*}
Further, we can use Pinsker's inequality to write:
$$\norm{T(\cdot|s,a)-\That{}{\cdot|s,a}}_{1}\leq \sqrt{2KL\big(T(\cdot|s,a)||\That{}{\cdot|s,a}\big)}\ .$$
This justifies the use of maximum likelihood estimation for model learning, a common practice in model-based RL \cite{bagnell2001autonomous,abbeel2006using,agostini2010reinforcement}, since maximum likelihood estimation is equivalent to empirical KL minimization.

However, there exists a major drawback with the KL objective, namely that it ignores the structure of the value function during model learning. As a simple example, if the value function is constant through the state-space, any randomly chosen model $\widehat T$ will, in fact, yield zero Bellman error. However, a model learning algorithm that ignores the structure of value function can potentially require many samples to provide any guarantee about the performance of learned policy.

Consider the objective function $l(T,\Thatn)$, and notice again that $v$ itself is not known so we cannot directly optimize for this objective. \citet{farahmand17} proposed to search for a model that results in lowest error given all possible value functions belonging to a specific class: 
\begin{equation}L(T,\hat T)(s,a)\!=\!\sup_{v\in \mathcal{F}}\!\Big|\!\int\!\big(\!T(s'\mid s,a)-\Thatn(s'\mid s,a)\big)v(s') ds' \Big|^2
\label{vaml_tractable}
\end{equation}
Note that minimizing this objective is shown to be tractable if, for example, $\mathcal{F}$ is restricted to the class of exponential functions. Observe that the VAML objective (\ref{vaml_tractable}) is similar to the dual of Wasserstein (\ref{KR_duality}), but the main difference is the space of value functions. In the next section we show that even the space of value functions are the same under certain conditions.
\section{Lipschitz Generalized Value Iteration}
 We show that solving for a class of Bellman equations yields a Lipschitz value function. Our proof is in the context of GVI \citep{littman1996generalized}, which defines Value Iteration \citep{bellman1957markovian} with arbitrary backup operators. We make use of the following lemmas.

\begin{restatable}{lemma}{mullerLemma}
	Given a non-expansion $f:\fancyS\mapsto \mathbb R$:
	$$K_{d_\fancyS,d_{\real}}^\fancyA \big(\int T(s'|s,a)f(s')ds'\big)\leq K^{\fancyA}_{d_\fancyS,W}\big(T\big)\ .$$
\label{lipschitz_transition}
\end{restatable}
\begin{proof}
Starting from the definition, we write:

$\begin{aligned}
	&K_{d_\fancyS,d_{\real}}^\fancyA \big(\int T(s'|s,a)f(s')ds'\big)\\
	&=\sup_{a}\sup_{s_1,s_2}\frac{\Big|\int \big( T(s'|s_1,a)- T(s'|s_2,a)\big)f(s')ds'\Big|}{d(s_1,s_2)}\\
	&\leq\sup_{a}\sup_{s_1,s_2}\frac{\Big|\sup_{g}\int\! \big( T(s'|s_1,a)- T(s'|s_2,a)\big)g(s')ds'\Big|}{d(s_1,s_2)}\\
	&(\textrm{where } K_{d_\fancyS,d_\real}(g)\leq 1)\\
	&=\sup_{a}\sup_{s_1,s_2}\frac{\sup_{g}\int \big( T(s'|s_1,a)- T(s'|s_2,a)\big)g(s')ds'}{d(s_1,s_2)}\\
	&=\sup_{a}\sup_{s_1,s_2}\frac{W\big( T(\cdot|s_1,a), T(\cdot|s_2,a)\big)}{d(s_1,s_2)}= K_{d_{\fancyS},W}^{\fancyA}(T) \ .
\end{aligned}$
\end{proof}
\begin{restatable}{lemma}{LipschitzOperators}
	The following operators are non-expansion ($K_{\norm{\cdot}_{\infty},d_{R}}(\cdot)=1$):
	\begin{enumerate}
		\item $\textrm{max}(x),\ \textrm{mean}(x)$
		\item $\epsilon$-$greedy(x):=\epsilon\ \textrm{mean}(x) +(1-\epsilon) \textrm{max}(x)$ 
		\item $mm_\beta(x):=\frac{\log\frac{\sum_{i}e^{\beta x_i}}{n}}{\beta}$
	\end{enumerate}
	\label{operators_Lipschitzness}
\end{restatable}
\begin{proof}
1 is proven by \citet{littman1996generalized}. 2 follows from 1: (metrics not shown for brevity)
\begin{eqnarray*}
K(\epsilon\textrm{-}\textrm{greedy(x)})&=&K\big(\epsilon\ \textrm{mean}(x) +(1-\epsilon)\textrm{max}(x)\big)\\
&\leq&\epsilon K\big( \textrm{mean}(x)\big) +(1-\epsilon)K\big(\textrm{max}(x)\big)\\
&=&1
\end{eqnarray*}
Finally, 3 is proven multiple times in the literature. \citep{mellowmax,nachum2017bridging,neu2017unified}
\end{proof}

\begin{algorithm}
\caption{GVI algorithm}
\begin{algorithmic}
   \STATE {\bfseries Input:} initial $\widehat Q(s,a)$, $\delta$, and choose an operator $f$
   \REPEAT 
   \STATE $\textrm{diff} \leftarrow 0$
   \FOR{each $s\in \mathcal{S}$}
   \FOR{each $a\in \mathcal{A}$}
   \STATE $Q_{copy}\leftarrow \widehat Q(s,a)$
   \STATE $\widehat Q(s,a)\! \leftarrow\! R(s,a)\!+\!\gamma\!\int T(s'\mid s,a) f\big(\widehat Q(s',\cdot)\big)ds'$ 
   \STATE $\textrm{diff} \leftarrow \max\big\{\textrm{diff},|Q_{copy}- Q(s,a)|\big\}$
   \ENDFOR
   \ENDFOR
   \UNTIL{$\textrm{diff}<\delta$}
\end{algorithmic}
\label{GVI}
\end{algorithm}

We now present the main result of this paper.

\noindent\fbox{%
    \parbox{.47\textwidth}{%
        \begin{theorem*}
	For any choice of backup operator $f$ outlined in Lemma~\ref{operators_Lipschitzness}, GVI computes a value function with a Lipschitz constant bounded by $\frac{K^{\fancyA}_{d_{\fancyS},d_{R}}(R)}{1-\gamma K_{d_{\fancyS},W}( T)}\ $ if $\gamma K_{d_{\fancyS},W}^{\fancyA}( T)< 1$.
	\label{theorem_lipschitz_q}
\end{theorem*}
    }%
}
\begin{proof}
From Algorithm~\ref{GVI}, in the $n$th round of GVI updates we have:
	$$\widehat Q_{n+1}(s,a) \leftarrow R(s,a)+\gamma\int T(s'\mid s,a) f\big(\widehat Q_{n}(s',\cdot)\big) ds'.$$
	First observe that:
	\begin{eqnarray*}
	&&\!K^{\fancyA}_{d_{\fancyS},d_{R}}(\widehat Q_{n+1})\\
	    && \textrm{\big(due to Summation Lemma (\ref{lemma_Lipschitz_summation})\big)}\\
		&&\leq \! K^{\fancyA}_{d_{\fancyS},d_{R}}\!(R)\!+\!\gamma K_{d_{\fancyS},d_\real}^{\fancyA}\!\big(\!\int\!T(s'\mid s,a)\! f\big( \widehat Q_{n}(s',\cdot)\big)ds'\big) \\
		&& \textrm{\big(due to Lemma (\ref{lipschitz_transition})\big)}\\
		&&\leq K^{\fancyA}_{d_{\fancyS},d_{R}}(R)+\gamma K_{d_{\fancyS},W}^{\fancyA}( T)\ K_{d_{\fancyS,\real}}\Big(f\big( \widehat Q_{n}(s,\cdot)\big)\Big)\\
		&&\textrm{\big(due to Composition Lemma (\ref{lemma_Lipschitz_Composition})\big)}\\
		&&\leq
		K^{\fancyA}_{d_{\fancyS},d_{R}}(R)+\gamma K_{d_{\fancyS},W}^{\fancyA}( T) K_{\norm{\cdot}_{\infty},d_{\real}}(f) K^{\fancyA}_{d_{\fancyS},d_\real}(\widehat Q_{n})\\
		&& \textrm{\big(due to Lemma (\ref{operators_Lipschitzness}), the non-expansion property of $f$\big)}\\
		&&=
		K^{\fancyA}_{d_{\fancyS},d_{R}}(R)+\gamma K_{d_{\fancyS},W}^{\fancyA}( T) K^{\fancyA}_{d_{\fancyS},d_\real}(\widehat Q_{n})
	\end{eqnarray*}
Equivalently:
\begin{eqnarray*}
K^{\fancyA}_{d_{\fancyS},d_{R}}(\widehat Q_{n+1})&\leq&
K^{\fancyA}_{d_{\fancyS},d_{\real}}(R)\!\sum_{i=0}^{n}\!\big(\gamma K_{d_{\fancyS},W}^{\fancyA}( T)\big)^i\\
&+&\big(\gamma K_{d_{\fancyS},W}^{\fancyA}(T)\big)^n \ K^{\fancyA}_{d_{\fancyS},d_{\real}}(\widehat Q_{0}) \ .
\end{eqnarray*}
By computing the limit of both sides, we get:
\begin{eqnarray*}\lim_{n\rightarrow\infty}K^{\fancyA}_{d_{\fancyS},d_{\real}}(\widehat Q_{n})\!&\!\leq\!&\!\lim_{n\rightarrow\infty}\! K^{\fancyA}_{d_{\fancyS},d_{\real}}(R)\!\sum_{i=0}^{n}\!\big(\gamma K_{d_{\fancyS},W}^{\fancyA}( T)\big)^i\\
&+&\lim_{n\rightarrow\infty}\big(\gamma K_{d_{\fancyS},W}^{\fancyA}(T)\big)^n \ K^{\fancyA}_{d_{\fancyS},d_{\real}}(\widehat Q_{0})\\
&=&\frac{K^{\fancyA}_{d_{\fancyS},d_{R}}(R)}{1-\gamma K_{d_{\fancyS},W}( T)} + 0 \ ,
\end{eqnarray*}
where we used the fact that
$$  \lim_{n\rightarrow\infty}\big(\gamma K_{d_{\fancyS},W}^{\fancyA}(T)\big)^n=0 \ .$$
This concludes the proof.\end{proof}
Now notice that as defined earlier:
$$\widehat V_{n}(s):=f\big(\widehat Q_{n}(s,\cdot)\big)\ ,$$ so as a relevant corollary of our theorem we get:
\begin{eqnarray*}
K_{d_{\fancyS},d_{\real}}\big(v(s)\big)&=&\lim_{n\rightarrow\infty} K_{d_{\fancyS},d_{\real}}(\widehat V_{n})\\
&=&\lim_{n\rightarrow\infty}K_{d_{\fancyS},d_{\real}}\Big(f\big(\widehat Q_{n}(s,\cdot)\big)\Big)\\
&\leq& \lim_{n\rightarrow\infty}K^{\fancyA}_{d_{\fancyS},d_{\real}}(\widehat Q_{n})\\
&\leq& \frac{K^{\fancyA}_{d_{\fancyS},d_{R}}(R)}{1-\gamma K_{d_{\fancyS},W}( T)} \ .
\end{eqnarray*}
That is, solving for the fixed point of this general class of Bellman equations results in a Lipschitz state-value function.
\section{Equivalence Between VAML and Wasserstein }
We now show the main claim of the paper, namely that minimzing for the VAML objective is the same as minimizing the Wasserstein metric. 

Consider again the VAML objective:
$$L(T,\hat T)(s,a)\!=\!\sup_{v\in \mathcal{F}}\Big|\!\int\!\big(T(s'\mid s,a)-\Thatn(s'\mid s,a)\big)v(s') ds' \Big|^2$$
where $\mathcal{F}$ can generally be any class of functions. From our theorem, however, the space of value functions $\mathcal{F}$ should be restricted to Lipschitz functions. Moreover, it is easy to design an MDP and a policy such that a desired Lipschitz value function is attained.

This space $\mathcal{L}_C$ can then be defined as follows:
$$\mathcal{L}_C=\{f:K_{d_{\fancyS},d_{\real}}(f)\leq C\}\ ,$$
where $$C=\frac{K^{\fancyA}_{d_{\fancyS},d_{R}}(R)}{1-\gamma K_{d_{\fancyS},W}( T)}\ . $$
So we can rewrite the VAML objective $L$ as follows:

$\begin{aligned}
&L\big(T,\hat T\big)(s,a)\!=\!\sup_{f\in \mathcal{L}_C}\!\Big|\!\int\!f(s)\big(T(s'\mid s,a)\!-\!\widehat T(s'\mid s,a)\big)\!ds' \Big|^2\\
&=\sup_{f\in \mathcal{L}_C}\Big|\!\int\!C\frac{f(s)}{C}\big(T(s'\mid s,a)\!-\!\widehat T(s'\mid s,a)\big)\ ds' \Big|^2\\
&=C^2\sup_{g\in \mathcal{L}_1}\Big|\!\int\! g(s)\big(T(s'\mid s,a)\!-\!\widehat T(s'\mid s,a)\big)\ ds' \Big|^2\ .\\
\end{aligned}$

It is clear that a function $g$ that maximizes the Kantorovich-Rubinstein dual form:

$\begin{aligned}
& \sup_{g\in \mathcal{L}_1}\int g(s)\big( T(s'\mid s,a)-\widehat T(s'\mid s,a)\big)\ ds'\\
&:=W(T(\cdot|s,a),\widehat{T}(\cdot|s,a)) \ , 
\end{aligned}$

will also maximize:
$$L\big(T,\hat T\big)(s,a)=\Big|\int g(s)\big(T(s'\mid s,a)-\widehat T(s'\mid s,a)\big)\ ds' \Big|^2\ .$$
This is due to the fact that $\forall g\in \mathcal{L}_1 \Rightarrow -g \in \mathcal{L}_1$ and so computing absolute value or squaring the term will not change $\textrm{arg} \max$ in this case.

As a result:
$$L\big(T,\widehat T\big)(s,a)=\Big(C\ W\big(T(\cdot|s,a),\widehat T(\cdot|s,a)\big)\Big)^2\ .$$
This highlights a nice property of Wasserstein, namely that minimizing this metric yields a value-aware model.
\section{Conclusion and Future Work}
We showed that the value function of an MDP is Lipschitz. This result enabled us to draw a connection between value-aware model-based reinforcement learning and the Wassertein metric. 

We hypothesize that the value function is Lipschitz in a more general sense, and so, further investigation of Lipschitz continuity of value functions should be interesting on its own. The second interesting direction relates to design of practical model-learning algorithms that can minimize Wasserstein. Two promising directions are the use of generative adversarial networks \cite{goodfellow2014generative,arjovsky2017wasserstein} or approximations such as entropic regularization \cite{wassApprox}. We leave these two directions for future work.

\bibliography{example_paper}
\bibliographystyle{icml2018}

\end{document}